\documentclass{article} 
\usepackage{nips15submit_e,times}
\usepackage{url}

\usepackage{hyperref}

\usepackage[square]{natbib}
\usepackage{cleveref,setspace}

\usepackage{times}

\usepackage{graphicx}
\usepackage{caption}
\usepackage{subcaption}

\usepackage{amsmath}
\usepackage{amsfonts}
\usepackage{amssymb}
\usepackage{amsthm}

\usepackage[utf8]{inputenc}
\usepackage{amsthm}
\usepackage{amsmath}
\usepackage{amsfonts}
\usepackage{amssymb}
\usepackage{dsfont}
\usepackage{color}
\allowdisplaybreaks
\newcommand{\R}{\mathbb{R}}
\newcommand{\N}{\mathcal{N}}
\newcommand{\cL}{\mathcal{L}}

\newcommand{\svert}{~|~}
\newcommand{\td}{\text{d}}
\newcommand{\f}{\mathbf{f}}
\newcommand{\x}{\mathbf{x}}
\newcommand{\Bb}{\mathbf{b}}
\newcommand{\sBb}{\mathtt{z}}

\newcommand{\y}{\mathbf{y}}

\newcommand{\w}{\mathbf{w}}
\newcommand{\W}{\mathbf{W}}

\newcommand{\m}{\mathbf{m}}

\newcommand{\bL}{\mathbf{L}}

\newcommand{\X}{\mathbf{X}}
\newcommand{\Y}{\mathbf{Y}}
\newcommand{\F}{\mathbf{F}}
\newcommand{\I}{\mathbf{I}}
\newcommand{\M}{\mathbf{M}}

\newcommand{\bz}{\mathbf{0}}
\newcommand{\bepsilon}{\text{\boldmath$\epsilon$}}

\newcommand{\bo}{\text{\boldmath$\omega$}}
\newcommand{\bsigma}{\text{\boldmath$\sigma$}}
\newcommand{\bSigma}{\text{\boldmath$\Sigma$}}
\newcommand{\bmu}{\text{\boldmath$\mu$}}
\newcommand{\bphi}{\text{\boldmath$\phi$}}
\newcommand{\K}{\mathbf{K}}
\newcommand{\Kh}{\widehat{\mathbf{K}}}
\newcommand{\Cov}{\text{Cov}}

\newcommand{\tr}{\text{tr}}

\newcommand{\diag}{\text{diag}}
\newcommand{\KL}{\text{KL}}

\newcommand{\bc}{\mathbf{c}}

\newcommand{\weightdecay}{\lambda}

\newtheorem{proposition}{Proposition}
\theoremstyle{definition}

\def\fl[#1\]{\begin{align}#1\end{align}}
\def\[#1\]{\begin{align*}#1\end{align*}}

\def\*[#1\]{\begin{align*}#1\end{align*}}

\usepackage[small,compact]{titlesec}
\titlespacing{\section}{0pt}{1ex}{0.8ex}
\titlespacing{\subsection}{0pt}{0.5ex}{0ex}
\titlespacing{\subsubsection}{0pt}{0.2ex}{0ex}
\expandafter\def\expandafter\normalsize\expandafter{%
    \normalsize
    \setlength\abovedisplayskip{2pt}
    \setlength\belowdisplayskip{2pt}
    \setlength\abovedisplayshortskip{0pt}
    \setlength\belowdisplayshortskip{0pt}
}

\title{Dropout as a Bayesian Approximation: \\
Appendix
}

\author{
\And
Yarin Gal 
\And
\\
University of Cambridge\\
\texttt{\{yg279,zg201\}@cam.ac.uk} 
\And
Zoubin Ghahramani
}

\nipsfinalcopy 

\begin{document}

\maketitle

\begin{abstract} 

We show that a neural network with arbitrary depth and non-linearities, with dropout applied before every weight layer, is mathematically equivalent to an approximation to a well known Bayesian model. This interpretation might offer an explanation to some of dropout's key properties, such as its robustness to over-fitting. Our interpretation allows us to reason about uncertainty in deep learning, and allows the introduction of the Bayesian machinery into existing deep learning frameworks in a principled way. 

This document is an appendix for the main paper ``Dropout as a Bayesian Approximation: 
Representing Model Uncertainty in Deep Learning'' by Gal and Ghahramani, 2015 (\url{http://arxiv.org/abs/1506.02142}).
\end{abstract} 

\section{Introduction}

Deep learning works very well in practice for many tasks, ranging from image processing \citep{krizhevsky2012imagenet} to language modelling \citep{bengio2006neural}. However the framework has some major limitations as well. Our inability to reason about uncertainty over the features is an example of such -- the features extracted from a dataset are often given as point estimates. These do not capture how much the model is confident in its estimation. 
On the other hand, probabilistic Bayesian models such as the Gaussian process \citep{Rasmussen2005Gaussian} offer us the ability to reason about our confidence. But these often come with a price of lessened performance. 

Another major obstacle with deep learning techniques is over-fitting. This problem has been largely answered with the introduction of dropout \citep{hinton2012improving,srivastava2014dropout}. Indeed many modern models use dropout to avoid over-fitting in practice. Over the last several years many have tried to explain why dropout helps in avoiding over-fitting, a property which is not often observed in \textit{Bayesian models}. Papers such as \citep{wager2013dropout,baldi2013understanding} have suggested that dropout performs stochastic
gradient descent on a regularised error function, or is equivalent to an $L_2$ regulariser applied after scaling the features by some estimate.

Here we show that a deep neural network (NN) with arbitrary depth and non-linearities, with dropout applied before every weight layer, is mathematically equivalent to an approximation to the probabilistic deep Gaussian process model \citep{damianou2013deep} (marginalised over its covariance function parameters). We would like to stress that no simplifying assumptions are made on the use of dropout in the literature, and that the results derived are applicable to any network architecture that makes use of dropout exactly as it appears in practical applications. We show that the dropout objective, in effect, minimises the Kullback--Leibler divergence between an approximate distribution and the posterior of a deep Gaussian process (marginalised over its finite rank covariance function parameters). 

We survey possible applications of this new interpretation, and discuss insights shedding light on dropout's properties. 
This interpretation of dropout as a Bayesian model offers an explanation to some of its properties, such as its ability to avoid over-fitting. Further, our insights allow us to treat NNs with dropout as fully Bayesian models, and obtain uncertainty estimates over their features. In practice, this allows the introduction of the Bayesian machinery into existing deep learning frameworks in a principled way.
Lastly, our analysis suggests straightforward generalisations of dropout for future research which should improve on current techniques.

The work presented here is an extensive theoretical treatment of the above, with applications studied separately. 
We next review the required background, namely dropout, Gaussian processes, and variational inference. We then derive the main results of the paper. We finish with insights and applications, and discuss how various dropout variants fit within our framework.

%

\section{Background}


We review dropout, the Gaussian process model\footnote{For a full treatment of Gaussian processes, see \citet{Rasmussen2005Gaussian}.}, and approximate variational inference quickly. These tools will be used in the following section to derive the main results of this work. We use the following notation throughout the paper. Bold lower case letters ($\x$) denote vectors, bold upper case letters ($\X$) denote matrices, and standard weight letters ($x$) denote scalar quantities. We use subscripts to denote either entire rows / columns (with bold letters, $\x_i$), or specific elements ($x_{ij}$). We use subscripts to denote variables as well (such as $\W_1: Q \times K, \W_2: K \times D$), with corresponding lower case indices to refer to specific rows / columns ($\w_q,\w_k$ for the first variable and $\w_k,\w_d$ for the second).
We use a second subscript to denote the element index of a specific variable: $w_{1,qk}$ denotes the element at row $q$ column $k$ of the variable $\W_1$.

\subsection{Dropout} \label{sect:dropout}
We review the dropout NN model \citep{hinton2012improving,srivastava2014dropout} quickly for the case of a \textit{single hidden layer} NN. This is done for ease of notation, and the generalisation to multiple layers is straightforward. Denote by $\W_1, \W_2$ the weight matrices connecting the first layer to the hidden layer and connecting the hidden layer to the output layer respectively. These linearly transform the layers' inputs before applying some element-wise non-linearity $\sigma(\cdot)$.
Denote by $\Bb$ the biases by which we shift the input of the non-linearity. 
We assume the model to output $D$ dimensional vectors while its input is $Q$ dimensional vectors, with $K$ hidden units. Thus $\W_1$ is a $Q \times K$ matrix, $\W_2$ is a $K \times D$ matrix, and $\Bb$ is a $K$ dimensional vector. A standard NN model would output $\widehat{\y} = \sigma(\x \W_1 + \Bb) \W_2$ given some input $\x$.\footnote{Note that we omit the outer-most bias term as this is equivalent to centring the output.}

Dropout is applied by sampling two binary vectors $\sBb_1, \sBb_2$ of dimensions $Q$ and $K$ respectively. The elements of the vectors are distributed according to a Bernoulli distribution with some parameter $p_i \in [0,1]$ for $i = 1,2$. Thus $\sBb_{1,q} \sim \text{Bernoulli}(p_1)$ for $q = 1,...,Q$, and $\sBb_{2,k} \sim \text{Bernoulli}(p_2)$ for $k = 1, ..., K$. Given an input $\x$, $1 - p_1$ proportion of the elements of the input are set to zero: $\x \circ \sBb_1$ where $\circ$ signifies the Hadamard product.
The output of the first layer is given by $\sigma((\x \circ \sBb_1)\W_1 + \Bb) \circ \sBb_2$, which is linearly transformed to give the dropout model's output $\widehat{\y} = \big( (\sigma((\x \circ \sBb_1)\W_1 + \Bb)) \circ\sBb_2 \big) \W_2$. This is equivalent to multiplying the weight matrices by the binary vectors to zero out entire rows: 
$$\widehat{\y} = \sigma(\x (\sBb_1 \W_1) + \Bb)(\sBb_2 \W_2).$$
The process is repeated for multiple layers. Note that to keep notation clean we will write $\sBb_1$ when we mean $\diag(\sBb_1)$ with the $\diag(\cdot)$ operator mapping a vector to a diagonal matrix whose diagonal is the elements of the vector. 

To use the NN model for regression we might use the Euclidean loss (also known as ``square loss''),
\begin{align} \label{eq:NN:reg}
E = \frac{1}{2N}
\sum_{n=1}^N ||\y_n - \widehat{\y}_n||^2_2
\end{align}
where $\{\y_1, \hdots, \y_N\}$ are $N$ observed outputs, and $\{\widehat{\y}_1, \hdots, \widehat{\y}_N\}$ being the outputs of the model with corresponding observed inputs $\{ \x_1, \hdots, \x_N \}$. 

To use the model for classification, predicting the probability of $\x$ being classified with label $1,...,D$, we pass the output of the model $\widehat{\y}$ through an element-wise softmax function to obtain normalised scores: $\widehat{p}_{nd} = \exp(\widehat{y}_{nd}) / \left(\sum_{d'} \exp(\widehat{y}_{nd'})\right)$. Taking the log of this function results in a \textit{softmax} loss,
\begin{align} \label{eq:NN:class}
E = -\frac{1}{N} \sum\limits_{n=1}^N \log(\widehat{p}_{n,c_n})
\end{align}
where $c_n \in [1, 2, ..., D]$ is the observed class for input $n$.

During optimisation a regularisation term is often added.
We often use $L_2$ regularisation weighted by some weight decay $\weightdecay$ (alternatively, the derivatives might be scaled), resulting in a minimisation objective (often referred to as cost),
\begin{align} \label{eq:L:dropout}
\cL_{\text{dropout}} := E + \weightdecay_1 ||\W_1||^2_2 + \weightdecay_2 ||\W_2||^2_2 + \weightdecay_3 ||\Bb||^2_2.
\end{align}
We sample new realisations for the binary vectors $\sBb_i$ for every input point and every forward pass thorough the model (evaluating the model's output), and use the same values in the backward pass (propagating the derivatives to the parameters). 


The dropped weights $\sBb_1\W_1$ and $\sBb_2\W_2$ are often scaled by $\frac{1}{p_i}$ to maintain constant output magnitude. At test time no sampling takes place. This is equivalent to initialising the weights $\W_i$ with scale $\frac{1}{p_i}$ with no further scaling at training time, and at test time scaling the weights $\W_i$ by $p_i$. 

We will show that equations \eqref{eq:NN:reg} to \eqref{eq:L:dropout} arise in Gaussian process approximation as well. But first, we introduce the Gaussian process model.

\subsection{Gaussian Processes}

The Gaussian process (GP) is a powerful tool in statistics that allows us to model distributions over functions. It has been applied in both the supervised and unsupervised domains, for both regression and classification tasks \citep{Rasmussen2005Gaussian,titsias2010bayesian,Gal2015Latent}. The Gaussian process offers desirable properties such as uncertainty estimates over the function values, robustness to over-fitting, and principled ways for hyper-parameter tuning. The use of \textit{approximate variational inference} for the model allows us to scale it to large data via stochastic and distributed inference \citep{hensman2013Gaussian,Gal2014DistributedB}.

Given a training dataset consisting of $N$ inputs $\{ \x_1, \hdots, \x_N \}$ and their corresponding outputs $\{\y_1, \hdots, \y_N\}$, we would like to estimate a function $\y = \f(\mathbf{x})$ that is likely to have generated our observations. We denote the inputs $\X \in \R^{N \times Q}$ and the outputs $\Y \in \R^{N \times D}$. 
 
What is a function that is likely to have generated our data? Following the Bayesian approach we would put some \textit{prior} distribution over the space of functions $p(\f)$. This distribution represents our prior belief as to which functions are more likely and which are less likely to have generated our data. We then look for the \textit{posterior} distribution over the space of functions given our dataset $(\X, \Y)$:
$$
p(\f | \X, \Y) \propto p(\Y | \X, \f) p(\f).
$$
This distribution captures the most likely functions given our observed data.

By modelling our distribution over the space of functions with a Gaussian process we can analytically evaluate its corresponding posterior in regression tasks, and estimate the posterior in classification tasks. 
In practice what this means is that for regression we place a joint Gaussian distribution over all function values,
\begin{align} \label{eq:generative_model_reg}
\F \svert \X &\sim \N(\bz, \K(\X, \X)) \\
\Y \svert \F &\sim \N(\F, \tau^{-1} \I_N) \notag
\end{align}
with some precision hyper-parameter $\tau$ and where $\I_N$ is the identity matrix with dimensions $N \times N$. For classification we sample from a categorical distribution with probabilities given by passing $\Y$ through an element-wise softmax, 
\begin{align} \label{eq:generative_model_class}
\F \svert \X &\sim \N(\bz, \K(\X, \X)) \\
\Y \svert \F &\sim \N(\F, 0 \cdot \I_N) \notag\\
c_n \svert \Y &\sim \text{Categorical}\left( \exp(y_{nd}) / \left(\sum_{d'} \exp(y_{nd'})\right) \right) \notag
\end{align}
for $n = 1, ..., N$ with observed class label $c_n$. Note that we did not simply write $\Y = \F$ because of notational convenience that will allow us to treat regression and classification together.

To model the data we have to choose a covariance function $\K(\X_1, \X_2)$ for the Gaussian distribution. This function defines the (scalar) similarity between every pair of input points $\K(\mathbf{x}_i, \mathbf{x}_j)$. Given a finite dataset of size $N$ this function induces an $N \times N$ covariance matrix which we will denote $\K := \K(\X, \X)$. For example we may choose a stationary squared exponential covariance function. We will see below that certain non-stationary covariance functions correspond to \textit{TanH} (hyperbolic tangent) or \textit{ReLU} (rectified linear) NNs. 


Evaluating the Gaussian distribution above involves an inversion of an $N$ by $N$ matrix, an operation that requires $\mathcal{O}(N^3)$ time complexity. Many approximations to the Gaussian process result in a manageable time complexity. 
Variational inference can be used for such, and will be explained next.



\subsection{Variational Inference}

To approximate the model above we could condition the model on a finite set of random variables $\bo$. We make a modelling assumption and assume that the model depends on these variables alone, making them into sufficient statistics in our approximate model.

The predictive distribution for a new input point $\x^*$ is then given by 
$$
p(\y^* | \x^*, \X, \Y) = \int p(\y^* | \x^*, \bo) p(\bo | \X, \Y)\ \td \bo,
$$
with $\y^* \in \R^{D}$.
The distribution $p(\bo | \X, \Y)$ cannot usually be evaluated analytically. Instead we define an approximating \textit{variational} distribution $q(\bo)$, whose structure is easy to evaluate.

We would like our approximating distribution to be as close as possible to the posterior distribution obtained from the full Gaussian process. We thus minimise the Kullback--Leibler (KL) divergence, intuitively a measure of similarity between two distributions:
\begin{align*}
\KL(q(\bo) ~|~ p(\bo | \X, \Y)),
\end{align*}
resulting in the approximate predictive distribution 
\begin{align} \label{eq:predictive_dist}
q(\y^* | \x^*) = \int p(\y^* | \x^*, \bo) q(\bo)\td \bo.
\end{align}

Minimising the Kullback--Leibler divergence is equivalent to maximising the \textit{log evidence lower bound} \citep{Bishop2006Pattern},
\begin{align}
&\cL_{\text{VI}} := \int q(\bo) \log p(\Y | \X, \bo) \td \bo - \KL(q(\bo) || p(\bo)) \label{eq:L:VI}
\end{align}
with respect to the variational parameters defining $q(\bo)$. Note that the KL divergence in the last equation is between the approximate posterior and the \textit{prior} over $\bo$. Maximising this objective will result in a variational distribution $q(\bo)$ that explains the data well (as obtained from the first term---the log likelihood) while still being close to the prior---preventing the model from over-fitting.

We next present a variational approximation to the Gaussian process extending on \citep{Gal2015Improving}, which results in a model mathematically identical to the use of dropout in arbitrarily structured NNs with arbitrary non-linearities.

\section{Dropout as a Bayesian Approximation}


We show that deep NNs with dropout applied before every weight layer are mathematically equivalent to approximate variational inference in the deep Gaussian process (marginalised over its covariance function parameters). For this we build on previous work \citep{Gal2015Improving} that applied variational inference in the \textit{sparse spectrum} Gaussian process approximation \citep{lazaro2010sparse}. Starting with the full Gaussian process we will develop an approximation that will be shown to be equivalent to the NN optimisation objective with dropout (eq.\ \eqref{eq:L:dropout}) with either the Euclidean loss (eq.\ \eqref{eq:NN:reg}) in the case of regression or softmax loss (eq.\ \eqref{eq:NN:class}) in the case of classification. This view of dropout will allow us to derive new  probabilistic results in deep learning. 

\subsection{A Gaussian Process Approximation} \label{sect:gp_approx_model}

We begin by defining our covariance function. Let $\sigma(\cdot)$ be some non-linear function such as the rectified linear (ReLU) or the hyperbolic tangent function (TanH).
We define $\K(\x, \y)$ to be 
$$
\K(\x, \y) = \int p(\w) p(b) \sigma(\w^T \x + b) \sigma(\w^T \y + b) \td \w \td b
$$
with $p(\w)$ a standard multivariate normal distribution of dimensionality $Q$ and some distribution $p(b)$. It is trivial to show that this defines a valid covariance function following \citep{tsuda2002marginalized}.

We use Monte Carlo integration with $K$ terms to approximate the integral above. This results in the finite rank covariance function
\begin{align*}
\Kh(\x, \y) &= 
\frac{1}{K} \sum_{k=1}^K 
 \sigma(\w_k^T \x + b_k) \sigma(\w_k^T \y + b_k)
\end{align*}
with $\w_k \sim p(\w)$ and $b_k \sim p(b)$. $K$ will be the number of hidden units in our single hidden layer NN approximation. 

Using $\Kh$ instead of $\K$ as the covariance function of the Gaussian process yields the following generative model: 
\begin{align} \label{eq:generative_model_approx}
\w_k &\sim p(\w), 
~b_k \sim p(b), \notag\\
\W_1 &= [\w_k]_{k=1}^K, \Bb = [b_k]_{k=1}^K \notag\\
\Kh(\x, \y) &= \frac{1}{K} \sum_{k=1}^K 
 \sigma(\w_k^T \x + b_k) \sigma(\w_k^T \y + b_k) \notag\\
\F \svert \X, \W_1, \Bb &\sim \N(\bz, \Kh(\X, \X)) \notag\\
\Y \svert \F &\sim \N(\F, \tau^{-1} \I_N),
\end{align}
with $\W_1$ a $Q \times K$ matrix parametrising our covariance function.

Integrating over the covariance function parameters results in the following predictive distribution:
\begin{align*}
p(\Y | \X) &= \int p(\Y | \F) p(\F | \W_1, \Bb, \X) p(\W_1) p(\Bb) 
\end{align*}
where the integration is with respect to $\F, \W_1$, and $\Bb$.

Denoting the $1 \times K$ row vector 
\begin{align*}
\bphi(\x, \W_1, \Bb) = \sqrt{\frac{1}{K}} \sigma(\W_1^T \x + \Bb) 
\end{align*}
and the $N \times K$ feature matrix $\Phi = [\bphi(\x_n, \W_1, \Bb)]_{n=1}^N$, we have $\Kh(\X,\X) = \Phi\Phi^T$.
We rewrite $p(\Y | \X)$ as 
\begin{align*}
&p(\Y | \X) = \int \N(\Y; \bz, \Phi\Phi^T + \tau^{-1} \I_N) p(\W_1)p(\Bb) \td \W_1 \td \Bb,
\end{align*}
analytically integrating with respect to $\F$.

The normal distribution of $\Y$ inside the integral above can be written as a joint normal distribution over $\y_d$, the $d$'th columns of the $N \times D$ matrix $\Y$, for $d = 1, ..., D$. 
For each term in the joint distribution, following identity \citep[page 93]{Bishop2006Pattern}, we introduce a $K \times 1$ auxiliary random variable $\w_d \sim \N(0, \I_K)$,
\begin{align}\label{eq:cond-prob-with-aux}
&\N(\y_d; 0, \Phi\Phi^T + \tau^{-1} \I_N) = \int \N(\y_d; \Phi\w_d, \tau^{-1} \I_N) \N(\w_d; 0, \I_K) \td \w_d.
\end{align}

Writing $\W_2 = [\w_d]_{d=1}^D$ a $K \times D$ matrix, the above is equivalent to\footnote{This is equivalent to the weighted basis function interpretation of the Gaussian process \citep{Rasmussen2005Gaussian} where the various quantities are analytically integrated over.} 
$$
p(\Y | \X) 
= \int p(\Y | \X, \W_1, \W_2, \Bb) p(\W_1) p(\W_2)p(\Bb) 
$$
where the integration is with respect to $\W_1, \W_2$, and $\Bb$.

We have re-parametrised the GP model and marginalised over the additional auxiliary random variables $\W_1, \W_2,$ and $\Bb$. We next approximate the posterior over these variables with appropriate approximating variational distributions.

\subsection{Variational Inference in the Approximate Model}
Our sufficient statistics are $\W_1, \W_2$, and $\Bb$.
To perform variational inference in our approximate model we need to define a variational distribution $q(\W_1, \W_2, \Bb) :=q (\W_1)q(\W_2)q(\Bb)$. We define $q(\W_1)$ to be a Gaussian mixture distribution with two components, factorised over $Q$:\footnote{Note that this is a bi-modal distribution defined over each output dimensionality; as a result the joint distribution over $\W_1$ is highly multi-modal.}
\begin{align} \label{eq:variaitonal_dist_1}
q(\W_1) &= \prod_{q=1}^Q q(\w_q), \\
q(\w_q) &= p_1 \N(\m_q, \bsigma^2 \I_K) + (1-p_1) \N(0, \bsigma^2 \I_K) \notag
\end{align}
with some probability $p_1 \in [0, 1]$, scalar $\bsigma > 0$ and $\m_q \in \R^K$. We put a similar approximating distribution over $\W_2$:
\begin{align} \label{eq:variaitonal_dist_2}
q(\W_2) &= \prod_{k=1}^K q(\w_k), \\
q(\w_k) &= p_2 \N(\m_k, \bsigma^2 \I_D) + (1-p_2) \N(0, \bsigma^2 \I_D) \notag
\end{align}
with some probability $p_2 \in [0, 1]$.

We put a simple Gaussian approximating distribution over $\Bb$:
\begin{align} \label{eq:variaitonal_dist_3}
q(\Bb) &= \N(\m, \bsigma^2 \I_K).
\end{align}

Next we evaluate the log evidence lower bound for the task of regression, for which we optimise over the variational parameters $\M_1 = [\m_q]_{q=1}^{Q}$, $\M_2 = [\m_k]_{k=1}^K$, and $\m$, to maximise Eq.\ \eqref{eq:L:VI}. The task of classification is discussed later.

\subsection{Evaluating the Log Evidence Lower Bound for Regression}

We need to evaluate the log evidence lower bound:
\begin{align} \label{eq:L:GP-VI}
\hspace{-1mm}
\cL_{\text{GP-VI}} := &\int q(\W_1, \W_2, \Bb) \log p(\Y | \X, \W_1, \W_2, \Bb) - \KL(q(\W_1, \W_2, \Bb) || p(\W_1, \W_2, \Bb)),
\end{align}
where the integration is with respect to $\W_1, \W_2$, and $\Bb$.

For the task of regression we can rewrite the integrand as a sum:
\begin{align*}
\log p(\Y | \X, {\W}_1, {\W}_2, {\Bb}) &= \sum_{d=1}^D \log \N(\y_d; \Phi{\w}_d, \tau^{-1} \I_N) \\
&= 
- \frac{ND}{2} \log(2\pi)
+ \frac{ND}{2} \log(\tau) 
- \sum_{d=1}^D \frac{\tau}{2} || \y_d - \Phi{\w}_d ||^2_2,
\end{align*}
as the output dimensions of a multi-output Gaussian process are assumed to be independent.
Denote $\widehat{\Y} = \Phi \W_2$. We can then sum over the rows instead of the columns of $\widehat{\Y}$ and write
$$\sum_{d=1}^D \frac{\tau}{2} || \y_d - \widehat{\y}_d ||^2_2 = \sum_{n=1}^N \frac{\tau}{2} || \y_n - \widehat{\y}_n ||^2_2.$$
Here $\widehat{\y}_n = \bphi(\x_n, {\W}_1, {\Bb}) {\W}_2 = \sqrt{\frac{1}{K}} \sigma(\x_n {\W}_1 + {\Bb}) {\W}_2$, resulting in the integrand
\begin{align*}
\log p(\Y | \X, {\W}_1, {\W}_2, {\Bb}) &= \sum_{n=1}^N \log \N(\y_n; \bphi(\x_n, {\W}_1, {\Bb}) {\W}_2, \tau^{-1} \I_D).
\end{align*}
This allows us to write the log evidence lower bound as
\begin{align*}
\hspace{-1mm}
\sum_{n=1}^N \int q(\W_1, \W_2, \Bb) \log p(\y_n | \x_n, \W_1, \W_2, \Bb) - \KL(q(\W_1, \W_2, \Bb) || p(\W_1, \W_2, \Bb)).
\end{align*}

We re-parametrise the integrands in the sum to not depend on $\W_1, \W_2$, and $\Bb$ directly, but instead on the standard normal distribution and the Bernoulli distribution. 
Let $q(\bepsilon_1) = \N(\bz, \I_{Q \times K})$ and $q(\sBb_{1,q}) = \text{Bernoulli}(p_1)$ for $q = 1, ..., Q$, and $q(\bepsilon_2) = \N(\bz, \I_{K \times D})$ and $q(\sBb_{2,k}) = \text{Bernoulli}(p_2)$ for $k = 1, ..., K$. Finally let $q(\bepsilon) = \N(0, \I_K)$. We write 
\begin{align} \label{eq:reparam_reg}
\W_1 &= \sBb_1 (\M_1 + \bsigma \bepsilon_1) + (1-\sBb_1) \bsigma \bepsilon_1 , \notag\\
\W_2 &= \sBb_2 (\M_2 + \bsigma \bepsilon_2) + (1-\sBb_2) \bsigma \bepsilon_2 , \notag\\
\Bb &= \m + \bsigma \bepsilon, 
\end{align}
allowing us to re-write the sum over the integrals in the above equation as 
\begin{align*}
&\sum_{n=1}^N \int q(\W_1, \W_2, \Bb) \log p(\y_n | \x_n, \W_1, \W_2, \Bb) \td \W_1 \td \W_2 \td \Bb \\
&\qquad = \sum_{n=1}^N \int q ( \sBb_1, \bepsilon_1, \sBb_2, \bepsilon_2, \bepsilon ) \log p(\y_n | \x_n, \W_1(\sBb_1, \bepsilon_1), \W_2 ( \sBb_2, \bepsilon_2 ), \Bb (\bepsilon)) 
\end{align*}
where each integration is over $\bepsilon_1, \sBb_1, \bepsilon_2, \sBb_2$, and $\bepsilon$.

We estimate each integral using Monte Carlo integration with a distinct single sample to obtain:
\begin{align*}
\cL_{\text{GP-MC}} := &\sum_{n=1}^N \log p(\y_n | \x_n, \widehat{\W}_1^n, \widehat{\W}_2^n, \widehat{\Bb}^n) - \KL(q(\W_1, \W_2, \Bb) || p(\W_1, \W_2, \Bb))
\end{align*}
with realisations $\widehat{\W}_1^n, \widehat{\W}_2^n, \widehat{\Bb}^n$ defined following eq.\ \eqref{eq:reparam_reg} with $\widehat{\bepsilon}_1^n \sim \N(\bz, \I_{Q \times K})$, $\widehat{\sBb}_{1,q}^n \sim \text{Bernoulli}(p_1)$, $\widehat{\bepsilon}_2^n \sim \N(\bz, \I_{K \times D})$, and $\widehat{\sBb}_{2,k}^n \sim \text{Bernoulli}(p_2)$.
Following \citep{blei2012variational,hoffman2013stochastic,kingma2013auto,rezende2014stochastic,titsias2014doubly}, optimising the \textit{stochastic} objective $\cL_{\text{GP-MC}}$ we would converge to the same limit as $\cL_{\text{GP-VI}}$. 

We can't evaluate the KL divergence term between a mixture of Gaussians and a single Gaussian analytically. However we can perform Monte Carlo integration like in the above. A further approximation for large $K$ (number of hidden units) and small $\bsigma^2$ yields a weighted sum of KL divergences between the mixture components and the single Gaussian (proposition \ref{prop:KL_mixture_of_Gaussians} in the appendix). Intuitively, this is because the entropy of a mixture of Gaussians with a large enough dimensionality and randomly distributed means tends towards the sum of the Gaussians' volumes.
Following the proposition, for large enough $K$ we can approximate the KL divergence term as 
\begin{align*}
&\KL(q(\W_1) || p(\W_1)) \approx QK(\bsigma^2 - \log(\bsigma^2) - 1) + \frac{p_1}{2} \sum_{q=1}^Q \m_q^T \m_q + C
\end{align*}
with $C$ constant w.r.t.\ our parameters, 
and similarly for $\KL(q(\W_2) || p(\W_2))$. The term $\KL(q(\Bb) || p(\Bb))$ can be evaluated analytically as
\begin{align*}
\KL(q(\Bb) || p(\Bb)) = \frac{1}{2} \big( \m^T\m + K(\bsigma^2 - \log(\bsigma^2) - 1) \big) + C
\end{align*}
with $C$ constant w.r.t.\ our parameters. We drop the constants for brevity.

Next we explain the relation between the above equations and the equations brought in section \ref{sect:dropout}.

\subsection{Log Evidence Lower Bound Optimisation}

Ignoring the constant terms $\tau, \bsigma$ we obtain the maximisation objective
\begin{align} \label{eq:L:GP-MC}
&\cL_{\text{GP-MC}} \propto - \frac{\tau}{2} \sum_{n=1}^N || \y_n - \widehat{\y}_n ||^2_2 - \frac{p_1}{2} ||\M_1||^2_2 - \frac{p_2}{2} ||\M_2||^2_2 - \frac{1}{2}||\m||^2_2.
\end{align}
Note that in the Gaussian processes literature the terms $\tau, \bsigma$ will often be optimised as well. 

Letting $\bsigma$ tend to zero, we get that the KL divergence of the prior blows-up and tends to infinity. However, in real-world scenarios setting $\bsigma$ to be machine epsilon ($10^{-33}$ for example in quadruple precision decimal systems) results in a constant value $\log \bsigma = -76$. With high probability samples from a standard Gaussian distribution with such a small standard deviation will be represented on a computer, in effect, as zero. Thus the random variable realisations $\widehat{\W}_1^n, \widehat{\W}_2^n, \widehat{\Bb}^n$ can be approximated as
\begin{align*}
\widehat{\W}_1^n \approx \widehat{\sBb}_1^n \M_1, ~~
\widehat{\W}_2^n \approx \widehat{\sBb}_2^n \M_2, ~~
\widehat{\Bb}^n \approx \m.
\end{align*}
Note that $\widehat{\W}_1^n$ are not maximum a posteriori (MAP) estimates, but random variable realisations.
This gives us 
\begin{align*}
\widehat{\y}_n \approx \sqrt{\frac{1}{K}} \sigma(\x_n (\widehat{\sBb}_1^n \M_1) + \m) (\widehat{\sBb}_2^n \M_2).
\end{align*}

Scaling the optimisation objective by a positive constant $\frac{1}{\tau N}$ doesn't change the parameter values at its optimum (as long as we don't optimise with respect to $\tau$). 
We thus scale the objective to get
\begin{align} \label{eq:L:GP-MC-reg}
\cL_{\text{GP-MC}} &\propto  - \frac{1}{2 N} \sum_{n=1}^N || \y_n - \widehat{\y}_n ||^2_2 - \frac{p_1}{2 \tau N} ||\M_1||^2_2 - \frac{p_2}{2 \tau N} ||\M_2||^2_2 - \frac{1}{2 \tau N}||\m||^2_2 
\end{align}
and we recovered equation \eqref{eq:NN:reg} for an appropriate setting of $\tau$. Maximising eq.\ \eqref{eq:L:GP-MC-reg} results in the same optimal parameters as the minimisation of eq.\ \eqref{eq:L:dropout}. Note that eq.\ \eqref{eq:L:GP-MC-reg} is a scaled unbiased estimator of eq.\ \eqref{eq:L:GP-VI}. With correct stochastic optimisation scheduling both will converge to the same limit.


The optimisation of $\cL_{\text{GP-MC}}$ proceeds as follows. We sample realisations $\widehat{\sBb}_1^n, \widehat{\sBb}_2^n$ to evaluate the lower-bound and its derivatives. We perform a single optimisation step (for example a single gradient descent step), and repeat, sampling new realisations.

We can make several interesting observations at this point. First, we can find the model precision from the identity $\weightdecay_1 = \frac{p_1}{2 \tau N}$ which gives $\tau = \frac{p_1}{2 \weightdecay_1 N}$. Second, it seems that the weight-decay for the dropped-out weights should be scaled by the probability of the weights to not be dropped. Lastly, it is known that setting the dropout probability to zero ($p_1 = p_2 = 1$) results in a standard NN. Following the derivation above, this would result in delta function approximating distributions on the weights (replacing eqs.\ \eqref{eq:variaitonal_dist_1}-\eqref{eq:variaitonal_dist_3}). As was discussed in \citep{lazaro2010sparse} this leads to model over-fitting. Empirically it seems that the Bernoulli approximating distribution is sufficient to considerably prevent over-fitting. 

Note that even though our approximating distribution is, in effect, made of a sum of two point masses, each point mass with zero variance, the mixture does not have zero variance. It has the variance of a Bernoulli random variable, which is transformed through the network. This choice of approximating distribution results in the dropout model.

\section{Extensions}

We have presented the derivation for a single hidden layer NN in the task of regression. The derivation above extends to tasks of classification, flexible priors, mini-batch optimisation, deep models, and much more. This will be studied next.

\subsection{Evaluating the Log Evidence Lower Bound for Classification} \label{sect:classification}
For classification we have an additional step in the generative model in eq.\ \eqref{eq:generative_model_class} compared to eq.\ \eqref{eq:generative_model_reg}, sampling class assignment $c_n$ given weight $\y_n$. We can write this generative model using the auxiliary random variables introduced in section \ref{sect:gp_approx_model} for the regression case by
\begin{align*}
p(\bc | \X) &= \int p(\bc | \Y) p(\Y | \X) \td \Y \\
&= \int p(\bc | \Y) \bigg( \int p(\Y | \X, \W_1, \W_2, \Bb) p(\W_1, \W_2, \Bb) \td \W_1 \td \W_2 \td \Bb \bigg) \td \Y
\end{align*}
where $\bc$ is an $N$ dimensional vector of categorical values.
We can write the log evidence lower bound in this case as (proposition \ref{prop:lower_bound_class} in the appendix)
\begin{align*}
\hspace{-1mm}
\cL_{\text{GP-VI}} := &\int p(\Y|\X, \W_1, \W_2, \Bb) q(\W_1, \W_2, \Bb) \log p(\bc | \Y) \td \W_1 \td \W_2 \td \Bb \td \Y  \notag\\
&\quad - \KL(q(\W_1, \W_2, \Bb) || p(\W_1, \W_2, \Bb)).
\end{align*}

The integrand of the first term can be re-written like before as a sum
\begin{align*}
&\log p(\bc | \widehat{\Y}) =
\sum\limits_{n=1}^N \log p(\bc_n | \widehat{\y}_n)
\end{align*}
resulting in a log evidence lower bound given by
\begin{align*}
\hspace{-1mm}
\cL_{\text{GP-VI}} := &\sum_{n=1}^N \int p(\y_n|\x_n, \W_1, \W_2, \Bb) q(\W_1, \W_2, \Bb) \log p(\bc_n | \y_n) 
\notag\\
&\qquad - \KL(q(\W_1, \W_2, \Bb) || p(\W_1, \W_2, \Bb))
\end{align*}
where the integration of each term in the first expression is over $\W_1, \W_2, \Bb$, and $\y_n$. 

We can re-parametrise each integrand in the sum following \eqref{eq:reparam_reg} to obtain 
\begin{align} \label{eq:reparam_class}
\W_1 &= \sBb_1 (\M_1 + \bsigma \bepsilon_1) + (1-\sBb_1) \bsigma \bepsilon_1 , \notag\\
\W_2 &= \sBb_2 (\M_2 + \bsigma \bepsilon_2) + (1-\sBb_2) \bsigma \bepsilon_2 , \notag\\
\Bb &= \m + \bsigma \bepsilon, \notag\\
\y_n &= \sqrt{\frac{1}{K}} \sigma(\x_n \W_1 + \Bb) \W_2.
\end{align}

Like before, we estimate each integral using Monte Carlo integration with a distinct single sample to obtain:
\begin{align*}
\cL_{\text{GP-MC}} := &\sum_{n=1}^N \log p(\bc_n | \widehat{\y}_n(\x_n, \widehat{\W}_1^n, \widehat{\W}_2^n, \widehat{\Bb}^n)) - \KL(q(\W_1, \W_2, \Bb) || p(\W_1, \W_2, \Bb))
\end{align*}
with realisations $\widehat{\y}_n, \widehat{\W}_1^n, \widehat{\W}_2^n$, and $\widehat{\Bb}^n$.

Each term in the sum in the first expression can be re-written as
\begin{align*}
&\log p(\bc_n | \widehat{\y}_n) =
\widehat{y}_{n c_n}
- \log \bigg( \sum_{d'} \exp(\widehat{y}_{nd'}) \bigg).
\end{align*}
We evaluate the second expression as before. Scaling the objective by a positive $\frac{1}{N}$ this results in the following maximisation objective,
\begin{align*}
&\cL_{\text{GP-MC}} \propto \frac{1}{N} \sum\limits_{n=1}^N 
\widehat{p}_{n, c_n}
- \frac{p_1}{2 N} ||\M_1||^2_2 - \frac{p_2}{2 N} ||\M_2||^2_2 - \frac{1}{2 N}||\m||^2_2,
\end{align*}
with $\widehat{p}_{n, c_n} = \log p(\bc_n | \widehat{\y}_n)$, identical (up to a sign flip) to that of eqs.\ \eqref{eq:NN:class}, \eqref{eq:L:dropout} for appropriate selection of weight decay $\weightdecay$.

\subsection{Prior Length-scale}
We can define a more expressive prior than $\N(0,I_K)$ over the weights of the first layer $\W_1$. This will allow us to incorporate our prior belief over the frequency of the observed data. Instead of $p(\w) = \N(\w; 0, I_K)$ we may use $p_l(\w) = \N(\w; 0, l^{-2} I_K)$ with length-scale $l$. 

To use this more expressive prior we need to adapt proposition \ref{prop:KL_mixture_of_Gaussians} approximating the prior term's KL divergence. It is easy to see that the KL divergence between $q(\x)$ and $p_l(\x)$ can be approximated as:
\begin{align*}
\KL(q(\x) || p_l(\x)) \approx
\sum_{i=1}^L \frac{p_i}{2} \big( l^{2} \bmu_i^T \bmu_i + \tr ( l^{2} \bSigma_i ) - K - \log | \bSigma_i | + K \log l^{-2} \big)
\end{align*}
plus a constant for large enough $K$. This follows from the KL divergence for multivariate normal distributions, and can be repeated for $\Bb$ as well (defining $p_{l'}(\Bb) = \N(\Bb; 0, l'^{-2} I_K)$). Following the previous derivations we obtain the regression objective
\begin{align*}
\cL_{\text{GP-MC}} &\propto  - \frac{1}{2 N} \sum_{n=1}^N || \y_n - \widehat{\y}_n ||^2_2 - \frac{l^2 p_1}{2 \tau N} ||\M_1||^2_2 - \frac{p_2}{2 \tau N} ||\M_2||^2_2 - \frac{l'^2}{2 \tau N}||\m||^2_2,
\end{align*}
and similarly for classification.

For high frequency data, setting $l$ and $l'$ to small values would result in a weaker regularisation over the weights $\M_1$ and $\m$. This leads to larger magnitude weights which can capture high frequency data \citep{Gal2015Improving}. 

The length-scale decouples the precision parameter $\tau$ from the weight-decays $\weightdecay$:
\begin{align}\label{eq:weightdecay_tau_relation}
\weightdecay_1 = \frac{l^2 p_1}{2 N \tau},
\end{align}
which results in
\begin{align}\label{eq:tau_weightdecay_relation}
\tau = \frac{l^2 p_1}{2 N \weightdecay_1}.
\end{align}

The length-scale is a user specified value that captures our belief over the function frequency. A short length-scale $l$ (corresponding to high frequency data) with high precision $\tau$ (equivalently, small observation noise) results in a small weight-decay $\lambda$ -- encouraging the model to fit the data well. A long length-scale with low precision results in a large weight-decay -- and stronger regularisation over the weights. This trade-off between the length-scale and model precision results in different weight-decay values. 

A similar term to the length-scale can be obtained for $\M_2$. Eq.\ \eqref{eq:cond-prob-with-aux} can be rewritten by substituting $\w_d' = \sqrt{\frac{1}{K}} \w_d$: we replace $\sqrt{\frac{1}{K}} \sigma(\W_1^T \x + \Bb) \w_d$ with $\sigma(\W_1^T \x + \Bb) \w_d'$ and replace $\N(\w_d; 0, \I_K)$ with $\N \big(\w_d'; 0, \frac{1}{K} \I_K \big)$. This results in the standard network structure often used in many implementations (without the additional scaling by $\sqrt{\frac{1}{K}}$). Following the above derivation, we can rewrite the KL divergence between $q(\x)$ and $\N (0, K^{-1} \I_K)$ to obtain the objective
\begin{align} \label{eq:L:GP-MC-reg-lengthscale}
\cL_{\text{GP-MC}} &\propto  - \frac{1}{2 N} \sum_{n=1}^N || \y_n - \widehat{\y}_n ||^2_2 - \frac{l^2 p_1}{2 \tau N} ||\M_1||^2_2 - \frac{K p_2}{2 \tau N} ||\M_2||^2_2 - \frac{l'^2}{2 \tau N}||\m||^2_2.
\end{align}
$K$ here acts in a similar way to $l$ and $l'$. A large number of hidden units results in a stronger regularisation, pushing the elements of $\M_2$ to become smaller and smaller. A smaller $K$ will result in a weaker regularisation term over $\M_2$, allowing its elements to take larger values.

\subsection{Mini-batch Optimisation}
We often use mini-batches when optimising eq.\ \eqref{eq:L:dropout}. This is done by setting eq.\ \eqref{eq:NN:reg} to
\begin{align*} 
E = \frac{1}{2M}
\sum_{n \in S} ||\y_n - \widehat{\y}_n||^2_2
\end{align*}
with a random subset of data points $S$ of size $M$, and similarly for classification.

Using recent results in stochastic variational inference \citep{hoffman2013stochastic} we can derive the equivalent to this in the GP approximation above. We change the likelihood term $-\frac{\tau}{2} \sum_{n=1}^N || \y_n - \widehat{\y}_n ||^2_2$ in eq.\ \eqref{eq:L:GP-MC} to sum over the data points in $S$ alone, and multiply the term by $\frac{N}{M}$ to get an unbiased estimator to the original sum.
Multiplying equation \eqref{eq:L:GP-MC} by $\frac{1}{\tau N}$ as before we obtain
\begin{align} \label{eq:L:GP-MC-reg-SVI}
\cL_{\text{GP-MC}} &\approx  - \frac{1}{2 M} \sum_{n \in S} || \y_n - \widehat{\y}_n ||^2_2 - \frac{p_1}{2 \tau N} ||\M_1||^2_2 - \frac{p_2}{2 \tau N} ||\M_2||^2_2 - \frac{1}{2 \tau N}||\m||^2_2 
\end{align}
recovering eq.\ \eqref{eq:L:dropout} for the mini-batch optimisation case as well.

\subsection{Predictive Log-likelihood}

Given a dataset $\X, \Y$ and a new data point $\x^*$ we can calculate the probability of possible output values $\y^*$ using the predictive probability $p(\y^* | \x^*, \X, \Y)$. The log of the predictive likelihood captures how well the model fits the data, with larger values indicating better model fit. 

Our predictive log-likelihood can be approximated by Monte Carlo integration of eq.\ \eqref{eq:predictive_dist} with $T$ terms:
\begin{align*} 
\log p(\y^* | \x^*, \X, \Y) &= \log \int p ( \y^* | \x^*, \bo ) p (\bo | \X, \Y) \td \bo \\
&\approx \log \int p ( \y^* | \x^*, \bo ) q (\bo) \td \bo \\
&\approx \log \bigg( \frac{1}{T} \sum_{t=1}^T p ( y^* | \x^*, \bo_t ) \bigg)
\end{align*}
with $\bo_t \sim q ( \bo \big)$. 

For regression we have 
\newcommand{\logsumexp}{\text{logsumexp}}
\begin{align} \label{eq:predictive_dist_reg}
\log p(\y^* | \x^*, \X, \Y) \approx \logsumexp \bigg( -\frac{1}{2} \tau || \y - \widehat{\y}_t ||^2 \bigg) - \log T - \frac{1}{2} \log 2 \pi - \frac{1}{2} \log \tau^{-1}
\end{align}
with our precision parameter $\tau$.

Uncertainty quality can be determined from this quantity as well.
Excessive uncertainty (large observation noise, or equivalently small model precision $\tau$) results in a large penalty from the last term in the predictive log-likelihood. 
An over-confident model with large model precision compared to poor mean estimation results in a penalty from the first term -- the distance $|| \y - \widehat{\y}_t ||^2$ gets amplified by $\tau$ which drives the first term to zero.

\subsection{Going Deeper than a Single Hidden Layer} \label{sect:going_deep}

We will demonstrate how to extend the derivation above to two hidden layers for the case of regression. Extension to further layers and classification is trivial.

We use the deep GP model -- feeding the output of one GP to the covariance of the next, in the same way the input is used in the covariance of the first GP.
However, to match the dropout NN model, we have to select a different covariance function for the GPs in the layers following the first one. For clarity, we denote here all quantities related to the first GP with subscript $1$, and as a second subscript denote the element index. So $\phi_{1,nk}$ denotes the element at row $n$ column $k$ of the variable $\Phi_1$, and $\bphi_{1,n}$ denotes row $n$ of the same variable.

We next define the new covariance function $\K_2$. Let $\sigma_2$ be some non-linear function, not necessarily the same as the one used with the previous covariance function. We define $\K_2(\x, \y)$ to be 
$$
\K_2(\x, \y) = \frac{1}{K_2} \int p(\Bb_2) \sigma_2(\x + \Bb_2)^T \sigma_2(\y + \Bb_2) \td \Bb_2
$$
with some distribution $p(\Bb_2)$ over $\Bb_2 \in \R^{K_1}$.

We use Monte Carlo integration with one term to approximate the integral above. This results in 
\begin{align*}
\Kh_2(\x, \y) &= 
\frac{1}{K_2} 
 \sigma(\x + \Bb_2)^T \sigma(\y + \Bb_2)
\end{align*}
with $\Bb_2 \sim p(\Bb_2)$.

Using $\Kh_2$ instead as the covariance function of the second Gaussian process yields the following generative model. First, we sample the variables for all covariance functions:
\begin{align*}
\w_{1,k} &\sim p(\w_1), 
~b_{1,k} \sim p(b_1), 
~\Bb_2 \sim p(\Bb_2) \notag\\
\W_1 &= [\w_{1,k}]_{k=1}^{K_1}, \Bb_1 = [b_{1,k}]_{k=1}^{K_1}
\end{align*}
with $\W_1$ a $Q \times K_1$ matrix, $\Bb_1$ a $K_1$ dimensional vector, and $\Bb_2$ a $K_2$ dimensional vector.
Given these variables, we define the covariance functions for the two GPs:
\begin{align*}
\Kh_1(\x, \y) &= \frac{1}{K_1} \sum_{k=1}^{K_1} 
 \sigma_1(\w_{1,k}^T \x + b_{1,k}) \sigma_1(\w_{1,k}^T \y + b_{1,k}) \notag\\
\Kh_2(\x, \y) &= \frac{1}{K_2} 
 \sigma(\x + \Bb_2)^T \sigma(\y + \Bb_2)
\end{align*}
Conditioned on these variables, we generate the model's output:
\begin{align*}
\F_1 \svert \X, \W_1, \Bb_1 &\sim \N(\bz, \Kh_1(\X, \X)) \notag\\
\F_2 \svert \X, \Bb_2 &\sim \N(\bz, \Kh_2(\F_1, \F_1)) \notag\\
\Y \svert \F_2 &\sim \N(\F_2, \tau^{-1} \I_N).
\end{align*}

We introduce auxiliary random variables $\W_2$ a $K_1 \times K_2$ matrix and $\W_3$ a $K_2 \times D$ matrix. The columns of each matrix distribute according to $\N(0, \I)$.

Like before, we write $\Kh_1(\X, \X) = \Phi_1 \Phi_1^T$ with $\Phi_1$ an $N \times K_1$ matrix and $\Kh_2(\X, \X) = \Phi_2 \Phi_2^T$ with $\Phi_2$ an $N \times K_2$ matrix:
\begin{align*}
\phi_{1,nk} &= \sqrt{\frac{1}{K_1}} \sigma_1(\w_{1,k}^T \x_n + b_{1,k}) \\
\phi_{2,nk} &= \sqrt{\frac{1}{K_2}} \sigma_2(f_{1,nk} + b_{2,k}).
\end{align*}
We can then write $\F_{1} = \Phi_1 \W_2$, since 
$$\mathbb{E}_{p(\W_2)}(\F_1) = \Phi_1\mathbb{E}_{p(\W_2)} (\W_2) = \bz$$ 
and 
$$\Cov_{p(\W_2)} (\F_1) = \mathbb{E}_{p(\W_2)} (\F_1 \F_1^T) = \Phi_1 \mathbb{E}_{p(\W_2)} (\W_2 \W_2^T) \Phi_1^T = \Phi_1 \Phi_1^T,$$
and similarly for $\F_2$. Note that $\F_1$ is an $N \times K_2$ matrix, and that $\F_2$ is an $N \times D$ matrix.
Thus, 
$$\phi_{2,nk} = \sqrt{\frac{1}{K_2}} \sigma_2(\w_{2,k}^T \bphi_{1,n} + b_{2,k}).$$

Finally, we can write
\begin{align*}
\y_n | \X, \W_1, \Bb_1, \W_2, \Bb_2, \W_3 \ \sim \N(\W_{3}^T \bphi_{2,n}, \tau^{-1} \I_{D}).
\end{align*}
The application of variational inference continues as before.

Note that an alternative approach would be to use the same covariance function in each layer. For that we would need to set $\W_2$ to be of dimensions $K_1 \times K_1$ normally distributed. This results in a product of two normally distributed matrices: $\W_2$ and the weights resulting from the Monte Carlo integration of $\Kh_2$ (denoted $\W_2'$ for convenience). Even though the composition of two linear transformations is a linear transformation, the resulting prior distribution over the weight matrix $\W_2\W_2'$ is quite complicated.

\section{Insights and Applications}

Our derivation suggests many applications and insights, including the representation of model uncertainty in deep learning, better model regularisation, computationally efficient Bayesian convolutional neural networks, use of dropout in recurrent neural networks, and the principled development of dropout variants, to name a few. These are \textit{briefly} discussed here, and studied more in depth in separate work.

\subsection{Insights}
The Gaussian process's robustness to over-fitting can be contributed to several different aspects of the model and is discussed in detail in \citep{Rasmussen2005Gaussian}. 
Our interpretation offers a possible explanation to dropout's ability to avoid over-fitting. Dropout can be seen as approximately integrating over the weights of the network.

Our derivation also suggests that an approximating variational distribution should be placed over the bias $\Bb$. This could be sampled jointly with the weights $\W$. Note that it is possible to interpret dropout as doing so when used with non-linearities with $\sigma(0)=0$. This is because the product by the vector of Bernoulli random variables can be passed through the non-linearity in this case. However the GP interpretation changes in this case, as the inputs are randomly set to zero rather than the weights. By sampling Bernoulli variables for the bias weights as well, the model might become more robust.

In \citep{srivastava2014dropout} alternative distributions to the Bernoulli are discussed. For example, it is suggested that multiplying the weights by $\N(1,\bsigma^2)$ results in similar results to dropout (although this becomes a more costly operation at run time). This can be seen as an alternative approximating variational distribution where we set $q(\w_k) = \m_k + \m_k \bsigma \bepsilon$ with $\bepsilon \sim \N(0, \I)$.

We noted in the text that the weight-decay for the dropped-out weights should be scaled by the probability of the weights to not be dropped. This follows from the KL approximation. 
We also note that the model brought in section \ref{sect:dropout} does not use a bias at the output layer. This is equivalent to shifting the data by a constant amount and thus not treated in our derivation. Alternatively, using a Gaussian process mean function given by $\mu(\x) = \Bb_L$ is equivalent to setting the bias of the output layer to $\Bb_L$.

\subsubsection{Model Calibration}
We can show that the dropout model is not calibrated. This is because Gaussian processes' uncertainty is not calibrated and the model draws its properties from these. The Gaussian process's uncertainty depends on the covariance function chosen, which we showed above to be equivalent to the non-linearities and prior over the weights. The choice of a GP's covariance function follows from our assumptions about the data. If we believe, for example, that the model's uncertainty should increase far from the data we might choose the squared exponential covariance function.

For many practical applications this means that model uncertainty can increase with data magnitude or be of different scale for different datasets. 
To calibrate model uncertainty in regression tasks we can scale the uncertainty linearly to remove data magnitude effects, and manipulate uncertainty percentiles to compare among different datasets. This can be done by fitting a simple distribution over the training set output uncertainty, and using the cumulative distribution function to find the relative ratio of a new data point's uncertainty to that of existing ones. This quantity can be used to compare a data point's uncertainty obtained from a model trained on one data distribution to another.

For example, if a test point has standard deviation $5$, whereas almost all other data points have standard deviation ranging from $0.2$ to $2$, then the data point will be in the top percentile, and the model will be considered as very uncertain about the point compared to the rest of the data. However, another model might give the same point standard deviation $5$ with most of the data modelled with standard deviation ranging from $10$ to $15$. In this model the data point will be in the lowest percentile, and the model will be considered as fairly certain about the point with respect to the rest of the data.

\subsection{Applications}

Our derivation suggests an estimate for dropout models by averaging $T$ forward passes through the network (referred to as \textit{MC dropout}, compared to \textit{standard dropout} with weight averaging).
This result has been presented in the literature before as model averaging \citep{srivastava2014dropout}. 
Our interpretation suggests a new look as to why MC dropout is more sensible than the current approach of averaging the weights. Furthermore, with the obtained samples we can estimate the model's confidence in its predictions and take actions accordingly. For example, in the case of classification, the model might return a result with high uncertainty, in which case we might decide to pass the input to a human to classify. Alternatively, one can use a weak and fast model to perform classification, and use a more elaborate but slower model only on inputs for which the weak model in uncertain. Uncertainty is important in reinforcement learning (RL) as well \citep{szepesvari2010algorithms}. With uncertainty information an agent can decide when to exploit and when to explore its environment. Recent advances in RL have made use of NNs to estimate agents' Q-value functions, a function that estimates the quality of different states and actions in the environment \citep{mnih2013playing}. Epsilon greedy search is often used in this setting, where an agent selects its currently estimated best action with some probability, and explores otherwise. With uncertainty estimates over the agent's Q-value function, techniques such as Thompson sampling \citep{thompson1933likelihood} can be used to train the model faster. These ideas are studied in the main paper.

Following our interpretation, one should apply dropout before each weight layer and not only before inner-product layers at the end of the model. This is to avoid parameter over-fitting on all layers as the dropout model, in effect, integrates over the parameters. The use of dropout before some layers but not others corresponds to interleaving MAP estimates and fully Bayesian estimates. The application of dropout before every weight layer is not used in practice however, as empirical results using \textit{standard} dropout on some network topologies (after convolutions for example) suggest inferior performance. The use of MC dropout, however, with dropout applied before every weight layer results in much better empirical performance on some NN structures. 

One can also interpret the approximation above as approximate variational inference in Bayesian neural networks (NNs). Thus, dropout applied before every weight layer is equivalent to variational inference in Bayesian NNs. This allows us to develop new Bayesian NN architectures which are not directly related to the Gaussian process, using operations such as pooling and convolutions. This leads to good, efficient, and trivial approximations to Bayesian convolutional neural networks (convnets). We discuss these ideas with empirical evaluation in separate work.

Another possible application is the adaptation of dropout to recurrent neural networks (RNNs). Currently, dropout is not used with these models as the repeated application of noise over potentially thousands of repetitions results in a very weak signal at the output. GP dynamical models \citep{wang2005gaussian} and recursive GPs with perfect integrators correspond to the ideas behind RNNs and long-short-term-memory (LSTM) networks \citep{hochreiter1997long}. The GP models integrate over the parameters and thus avoid over-fitting. Seen as a GP approximation one would expect there to exist a suitable dropout approximation for these tasks as well. We discuss these ideas in separate work.



In future research we aim to assess model uncertainty on adversarial inputs as well, such as corrupted images that classify incorrectly with high confidence \citep{szegedy2013intriguing}. Adding or subtracting a single pixel from each input dimension is perceived as almost unchanged input to a human eye, but can change classification probabilities considerably. In the high dimensional input space the new corrupted image lies far from the data, and one would expect model uncertainty to increase for such inputs.

Lastly, our interpretation allows the development of principled extensions of dropout. The use of non-diminishing $\bsigma^2$ (eqs.\ \eqref{eq:variaitonal_dist_1} to \eqref{eq:variaitonal_dist_3}) and the use of a mixture of Gaussians with more than two components is an immediate example of such. For example the use of a low rank covariance matrix would allow us to capture complex relations between the weights. These approximations could result in alternative uncertainty estimates to the ones obtained with MC dropout. This is subject to current research.

\section{Conclusions}

We have shown that a neural network with arbitrary depth and non-linearities and with dropout applied before every weight layer is mathematically equivalent to an approximation to the deep Gaussian process (marginalised over its covariance function parameters). 
This interpretation offers an explanation to some of dropout's key properties. Our analysis suggests straightforward generalisations of dropout for future research which should improve on current techniques.

%
%


\bibliography{example_paper}
\bibliographystyle{apalike}

\newpage
\appendix

\section{KL of a Mixture of Gaussians}
\begin{proposition} \label{prop:KL_mixture_of_Gaussians}
Fix $K, L \in \mathbb{N}$, a probability vector $\mathbf{p}=(p_1, ..., p_L)$, and $\bSigma_i \in \mathbb{R}^{K \times K}$ positive-definite for $i=1, ..., L$, with the elements of each $\bSigma_i$ not dependent on $K$. 
Let 
\begin{align*}
q(\x) = \sum_{i=1}^L p_i \N(\x; \bmu_i, \bSigma_i)
\end{align*}
be a mixture of Gaussians with $L$ components and $\bmu_i \in \R^K$ normally distributed, and let $p(\x) = \N(0, \I_K)$.

The KL divergence between $q(\x)$ and $p(\x)$ can be approximated as:
\begin{align*}
\KL(q(\x) || p(\x)) \approx
\sum_{i=1}^L \frac{p_i}{2} \big( \bmu_i^T \bmu_i + \tr(\bSigma_i) - K(1+\log 2 \pi) - \log | \bSigma_i | \big)
\end{align*}
plus a constant for large enough $K$.
\end{proposition}
\begin{proof}
We have
\begin{align*}
\KL(q(\x) || p(\x)) &= \int q(\x) \log \frac{q(\x)}{p(\x)} \td \x \\
&= \int q(\x) \log q(\x) \td \x - \int q(\x) \log p(\x) \td \x \\
&= - \mathcal{H}(q(\x)) - \int q(\x) \log p(\x) \td \x
\end{align*}
where $\mathcal{H}(q(\x))$ is the entropy of $q(\x)$. The second term in the last line can be evaluated analytically, but the entropy term has to be approximated.

We begin by approximating the entropy term. We write
\begin{align*}
\mathcal{H}(q(\x)) &= - \sum_{i=1}^L p_i \int \N(\x; \bmu_i, \bSigma_i) \log q(\x) \td \x \\
&= - \sum_{i=1}^L p_i \int \N(\bepsilon_i; 0, \I) \log q(\bmu_i + \bL_i \bepsilon_i) \td \bepsilon_i
\end{align*}
with $\bL_i \bL_i^T = \bSigma_i$.

Now, the term inside the logarithm can be written as 
\begin{align*}
q(\bmu_i + \bL_i \bepsilon_{i}) &= \sum_{j=1}^L p_i \N(\bmu_i + \bL_i \bepsilon_{i}; \bmu_j, \bSigma_j) \\
&= \sum_{j=1}^L p_i (2 \pi)^{-K/2} |\bSigma_j|^{-1/2} \exp \big\{ -\frac{1}{2} || \bmu_j - \bmu_i - \bL_i \bepsilon_{i} ||^2_{\bSigma_j} \big\}.
\end{align*}
where $|| \cdot ||_{\bSigma}$ is the Mahalanobis distance.
Since $\bmu_i,\bmu_j$ are assumed to be normally distributed, the quantity $\bmu_j - \bmu_i - \bL_i \bepsilon_{i}$ is also normally distributed. Using the expectation of the generalised $\chi^2$ distribution with $K$ degrees of freedom, we have that for $K >> 0$ there exists that $|| \bmu_j - \bmu_i - \bL_i \bepsilon_{i} ||^2_{\bSigma_j} >> 0$ for $i \neq j$ (since the elements of $\bSigma_j$ do not depend on $K$).
Finally, we have for $i = j$ that $|| \bmu_i - \bmu_i - \bL_i \bepsilon_{i} ||^2_{\bSigma_i} = \bepsilon_{i}^T \bL_i^T \bL_i^{-T} \bL^{-1}_i \bL_i \bepsilon_{i} = \bepsilon_{i}^T \bepsilon_{i}$. Therefore the last equation can be approximated as
\begin{align*}
&q(\bmu_i + \bL_i \bepsilon_{i}) 
\approx p_i (2 \pi)^{-K/2} |\bSigma_i|^{-1/2} \exp \big\{ -\frac{1}{2} \bepsilon_{i}^T \bepsilon_{i} \big\}.
\end{align*}
This gives us 
\begin{align*}
\mathcal{H}(q(\x)) &\approx- \sum_{i=1}^L p_i \int \N(\bepsilon_i; 0, \I)\log \bigg( p_i (2 \pi)^{-K/2} |\bSigma_i|^{-1/2} \exp \big\{ -\frac{1}{2} \bepsilon_{i}^T \bepsilon_{i} \big\} \bigg) \td \bepsilon_i \\
&= \sum_{i=1}^L \frac{p_i}{2} \bigg( \log |\bSigma_i| + \int \N(\bepsilon_i; 0, \I) \bepsilon_{i}^T \bepsilon_{i} \td \bepsilon_i + K \log 2 \pi \bigg)
+ C
\end{align*}
where $C = - \sum_{i=1}^L p_i \log p_i$.
Since $\bepsilon_{i}^T \bepsilon_{i}$ distributes according to a $\chi^2$ distribution, its expectation is $K$, and the entropy can be approximated as 
\begin{align*}
\mathcal{H}(q(\x)) \approx \sum_{i=1}^L \frac{p_i}{2} \big( \log |\bSigma_i| + K (1+\log 2 \pi) \big) + C
\end{align*}

Next, evaluating the first term of the KL divergence we get
\begin{align*}
\int q(\x) \log p(\x) \td \x &= \sum_{i=1}^L p_i \int  \N(\x; \bmu_i, \bSigma_i) \log p(\x) \td \x
\end{align*}
for $p(\x) = \N(0, \I_K)$ it is easy to validate that this is equivalent to
$-\frac{1}{2} \sum_{i=1}^L p_i \big( \bmu_i^T \bmu_i + \tr ( \bSigma_i ) \big)$.

Finally, we get 
\begin{align*}
\KL(q(\x) || p(\x)) \approx
\sum_{i=1}^L \frac{p_i}{2} \big( \bmu_i^T \bmu_i + \tr(\bSigma_i) - K(1+\log 2 \pi) - \log | \bSigma_i | \big) - C.
\end{align*}
\end{proof}

\section{Log Evidence Lower Bound for Classification}
\begin{proposition} \label{prop:lower_bound_class}
Given 
\begin{align*}
p(\bc | \X) &= \int p(\bc | \Y) p(\Y | \X) \td \Y \\
&= \int p(\bc | \Y) \bigg( \int p(\Y | \X, \W_1, \W_2, \Bb) \\
&\qquad\qquad\qquad\quad \cdot p(\W_1, \W_2, \Bb) \td \W_1 \td \W_2 \td \Bb \bigg) \td \Y
\end{align*}
where $\bc$ is an $N$ dimensional vector of categorical values, one for each observation, we can write the log evidence lower bound as
\begin{align*}
\hspace{-1mm}
\cL_{\text{GP-VI}} := &\int p(\Y|\X, \W_1, \W_2, \Bb) q(\W_1, \W_2, \Bb) \notag\\
&\qquad \qquad \cdot \log p(\bc | \Y) \td \W_1 \td \W_2 \td \Bb \td \Y  \notag\\
&\quad - \KL(q(\W_1, \W_2, \Bb) || p(\W_1, \W_2, \Bb)).
\end{align*}
\end{proposition}
\begin{proof}
We have 
\begin{align*}
&\log p(\bc | \X) \\
&= \log \int p(\bc | \Y) p(\Y | \X, \W_1, \W_2, \Bb) \\
&\qquad\qquad\qquad\quad \cdot p(\W_1, \W_2, \Bb) \td \W_1 \td \W_2 \td \Bb \td \Y \\
&= \log \int q(\W_1, \W_2, \Bb) p(\Y | \X, \W_1, \W_2, \Bb) p(\bc | \Y) \\
&\qquad\qquad\qquad\quad \cdot \frac{p(\W_1, \W_2, \Bb)}{q(\W_1, \W_2, \Bb)} \td \W_1 \td \W_2 \td \Bb \td \Y \\
&\geq \int q(\W_1, \W_2, \Bb) p(\Y | \X, \W_1, \W_2, \Bb) \log \bigg( p(\bc | \Y) \\
&\qquad\qquad\qquad\quad \cdot \frac{p(\W_1, \W_2, \Bb)}{q(\W_1, \W_2, \Bb)} \bigg) \td \W_1 \td \W_2 \td \Bb \td \Y \\
&= \int q(\W_1, \W_2, \Bb) p(\Y|\X, \W_1, \W_2, \Bb) \notag\\
&\qquad \qquad \cdot \log p(\bc | \Y) \td \W_1 \td \W_2 \td \Bb \td \Y  \notag\\
&\quad - \KL(q(\W_1, \W_2, \Bb) || p(\W_1, \W_2, \Bb)),
\end{align*}
as needed.
\end{proof}

\section{Predictive Mean}
\begin{proposition} \label{prop:mean}
Given weights matrices $\M_i$ of dimensions $K_i \times K_{i-1}$, bias vectors $\m_i$ of dimensions $K_i$, and binary vectors $\sBb_i$ of dimensions $K_{i-1}$ for each layer $i = 1, ..., L$, as well as the approximating variational distribution
\begin{align*}
q(\y^* | \x^*) := \N \big( \y^*; \widehat{\y}^*(\x^*, \sBb_1, ..., \sBb_L), \tau^{-1} \I_D \big) \text{Bern}(\sBb_1) \cdots \text{Bern}(\sBb_L)
\end{align*}
for some $\tau > 0$, with 
\begin{align*}
\widehat{\y}^* = \sqrt{\frac{1}{K_L}} (\M_L \sBb_L) \sigma \bigg( ... \sqrt{\frac{1}{K_1}} (\M_2 \sBb_2) \sigma \big( (\M_1 \sBb_1) \x^* + \m_1 \big) ... \bigg),
\end{align*}
we have 
\begin{align*}
\mathbb{E}_{q(\y^* | \x^*)} (\y^*) \approx \frac{1}{T} \sum_{t=1}^T \widehat{\y}^*(\x^*, \widehat{\sBb}_{1,t}, ..., \widehat{\sBb}_{L,t})
\end{align*}
with 
\begin{align*}
\widehat{\sBb}_{i,t} \sim \text{Bern}(p_i).
\end{align*}
\end{proposition}
\begin{proof}
\begin{align*}
\mathbb{E}_{q(\y^* | \x^*)} (\y^*) &= \int \y^* q(\y^* | \x^*) \td \y^* \\
&= \int \y^* \N \big( \y^*; \widehat{\y}^*(\x^*, \sBb_1, ..., \sBb_L), \tau^{-1} \I_D \big) \text{Bern}(\sBb_1) \cdots \text{Bern}(\sBb_L) \td \sBb_1 \cdots \td \sBb_L \td \y^* \\
&= \int \bigg( \int \y^* \N \big( \y^*; \widehat{\y}^*(\x^*, \sBb_1, ..., \sBb_L), \tau^{-1} \I_D \big) \td \y^* \bigg) \text{Bern}(\sBb_1) \cdots \text{Bern}(\sBb_L) \td \sBb_1 \cdots \td \sBb_L \td \y^* \\
&= \int \widehat{\y}^*(\x^*, \sBb_1, ..., \sBb_L) \text{Bern}(\sBb_1) \cdots \text{Bern}(\sBb_L) \td \sBb_1 \cdots \td \sBb_L \\
&\approx \frac{1}{T} \sum_{t=1}^T \widehat{\y}^*(\x^*, \widehat{\sBb}_{1,t}, ..., \widehat{\sBb}_{L,t}).
\end{align*}
\end{proof}

\section{Predictive Variance}
\begin{proposition} \label{prop:var}
Given weights matrices $\M_i$ of dimensions $K_i \times K_{i-1}$, bias vectors $\m_i$ of dimensions $K_i$, and binary vectors $\sBb_i$ of dimensions $K_{i-1}$ for each layer $i = 1, ..., L$, as well as the approximating variational distribution
\begin{align*}
q(\y^* | \x^*) &:= p(\y^* | \x^*, \bo) q(\bo) \\
q(\bo) &= \text{Bern}(\sBb_1) \cdots \text{Bern}(\sBb_L) \\
p(\y^* | \x^*, \bo) &= \N \big( \y^*; \widehat{\y}^*(\x^*, \sBb_1, ..., \sBb_L), \tau^{-1} \I_D \big) 
\end{align*}
for some $\tau > 0$, with 
\begin{align*}
\widehat{\y}^* = \sqrt{\frac{1}{K_L}} (\M_L \sBb_L) \sigma \bigg( ... \sqrt{\frac{1}{K_1}} (\M_2 \sBb_2) \sigma \big( (\M_1 \sBb_1) \x^* + \m_1 \big) ... \bigg),
\end{align*}
we have 
\begin{align*}
\mathbb{E}_{q(\y^* | \x^*)} \big( (\y^*)^T(\y^*) \big) \approx \tau^{-1} \I_D + \frac{1}{T} \sum_{t=1}^T \widehat{\y}^*(\x^*, \widehat{\sBb}_{1,t}, ..., \widehat{\sBb}_{L,t})^T \widehat{\y}^*(\x^*, \widehat{\sBb}_{1,t}, ..., \widehat{\sBb}_{L,t})
\end{align*}
with 
\begin{align*}
\widehat{\sBb}_{i,t} \sim \text{Bern}(p_i).
\end{align*}
\end{proposition}
\begin{proof}
\begin{align*}
&\mathbb{E}_{q(\y^* | \x^*)} \big( (\y^*)^T(\y^*) \big) \\
&= \int \bigg( \int (\y^*)^T(\y^*) p(\y^* | \x^*, \bo) \td \y^* \bigg) q(\bo) \td \bo \\ 
&= \int \bigg( \Cov_{p(\y^* | \x^*, \bo)}(\y^*) + \mathbb{E}_{p(\y^* | \x^*, \bo)} (\y^*)^T \mathbb{E}_{p(\y^* | \x^*, \bo)} (\y^*) \bigg) q(\bo) \td \bo \\ 
&= \int \bigg( \tau^{-1} \I_D + \widehat{\y}^*(\x^*, \sBb_1, ..., \sBb_L)^T \widehat{\y}^*(\x^*, \sBb_1, ..., \sBb_L) \bigg) \text{Bern}(\sBb_1) \cdots \text{Bern}(\sBb_L) \td \sBb_1 \cdots \td \sBb_L \\
&\approx \tau^{-1} \I_D + \frac{1}{T} \sum_{t=1}^T \widehat{\y}^*(\x^*, \widehat{\sBb}_{1,t}, ..., \widehat{\sBb}_{L,t})^T \widehat{\y}^*(\x^*, \widehat{\sBb}_{1,t}, ..., \widehat{\sBb}_{L,t})
\end{align*}
since $p(\y^* | \x^*, \bo) = \N \big( \y^*; \widehat{\y}^*(\x^*, \sBb_1, ..., \sBb_L), \tau^{-1} \I_D \big)$.
\end{proof}

\section{Detailed Experiment Set-up}

\subsection{Model Uncertainty in Regression Tasks -- Extrapolation}

We ran a stochastic gradient descent optimiser for 1,000,000 iterations (until convergence) with learning rate policy $\text{base-lr} * (1 + \gamma * \text{iter})^{-p}$ with $\gamma=0.0001, p=0.25$ and momentum $0.9$. We initialise the bias at 0 and initialise the weights uniformly from $[-\sqrt{3 / \text{fan-in}},\sqrt{3 / \text{fan-in}}]$. We use no mini-batch optimisation as the data is fairly small and with high frequencies. The learning rates used are $0.01$ with weight decay of $1e^{-06}$ for CO$_2$ (corresponding to a high noise precision of $1e^{5}$). This is to model the low observation noise in the data due to the scaling and high frequencies.

\subsection{Model Uncertainty in Reinforcement Learning}

For the purpose of this experiment, we used future rewards discount of $0.7$, no temporal window, and an experience replay of 30,000. The network starts learning after 1,000 steps, where in the initial 5,000 steps random actions are performed.
The networks consist of two ReLU hidden layers of size 50, with a learning rate and weight decay of 0.001. Stochastic gradient descent was used with no momentum and batch size of 64.

The original implementation makes use of epsilon greedy exploration with epsilon changing as 
$$
\epsilon = \min \left(1, \max \left(\epsilon_{min}, 1 - \frac{\text{age} - \text{burn-in}}{\text{steps-total} - \text{burn-in}} \right) \right)
$$
with steps-total of 200,000, burn-in of 3,000, and $\epsilon_{min}=0.05$. 

\subsection{Model Uncertainty in Regression Tasks -- Interpolation}

For interpolation we repeat the experiment in the main paper with ReLU networks with 5 hidden layers and the same setup on a new dataset -- solar irradiance. We use base learning rate of $5e^{-3}$ and weight decay of $5e^{-7}$. 

\begin{figure}[h]
	\centering
	\begin{subfigure}[b]{\textwidth}
		\includegraphics[width=\linewidth, height=2cm, trim=4mm 3mm 2mm 2mm, clip]{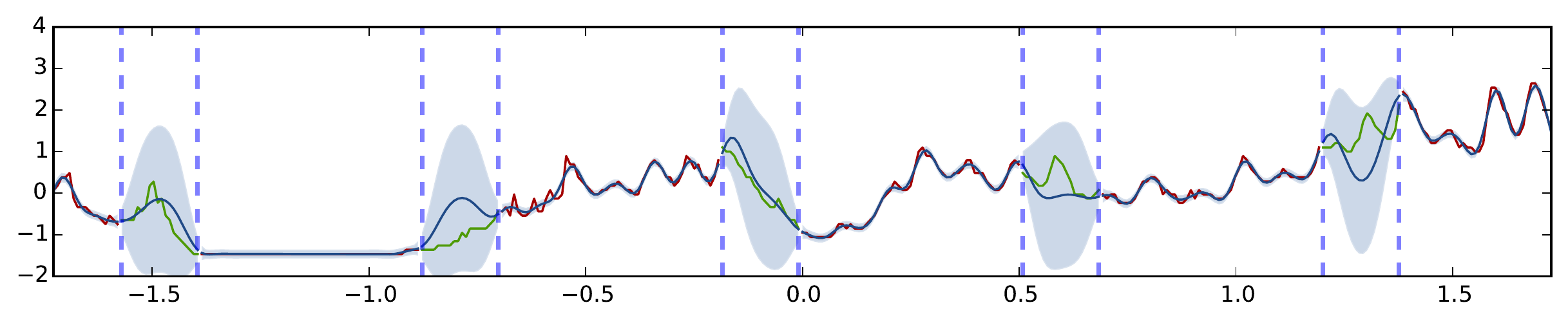}
		\vspace{-5mm}
		\caption{Gaussian process with SE covariance function} \label{fig:inter_GP}
	\end{subfigure}
	\begin{subfigure}[b]{\textwidth}
		\includegraphics[width=\linewidth, height=2cm, trim=4mm 3mm 2mm 2mm, clip]{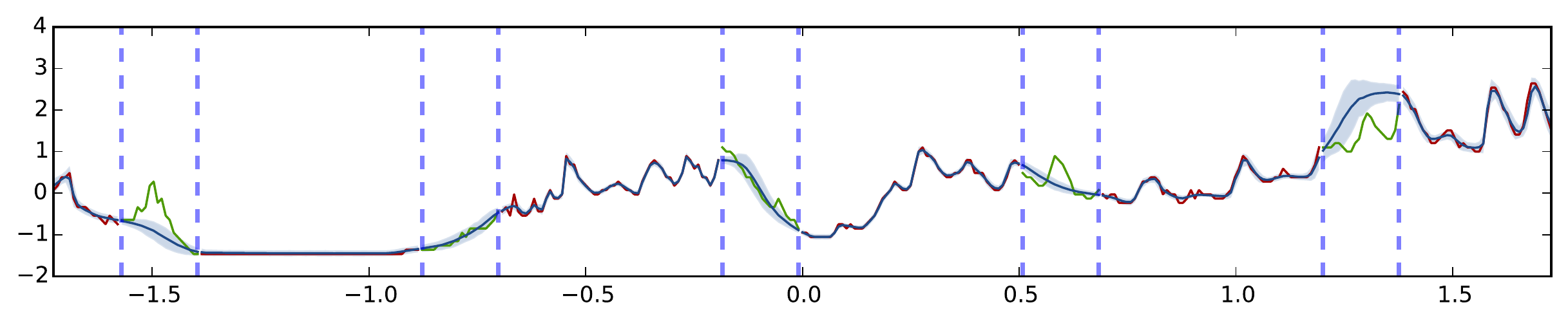}
		\vspace{-5mm}
		\caption{MC dropout with ReLU non-linearities} \label{fig:inter_MC}
	\end{subfigure}
	\caption{\textbf{Predictive mean and uncertainties on the reconstructed solar irradiance dataset with missing segments, for the GP and MC dropout approximation.} In red is the observed function and in green are the missing segments. In blue is the predictive mean plus/minus two standard deviations of the various approximations.} \label{fig:inter}
\end{figure}

Interpolation results are shown in fig.\ \ref{fig:inter}. Fig.\ \ref{fig:inter_GP} shows interpolation of missing sections (bounded between pairs of dashed blue lines) for the Gaussian process with squared exponential covariance function, as well the function value on the training set. In red is the observed function, in green are the missing sections, and in blue is the model predictive mean. Fig.\ \ref{fig:inter_MC} shows the same for the ReLU dropout model with 5 layers. 

Both models interpolate the data well, with increased uncertainty over the missing segments. However the GP's uncertainty is larger and over-estimates its 95\% confidence interval in capturing the true function. The observed uncertainty in MC dropout is similar to that of \citep{Gal2015Improving} with the model over-confident in its predictions.
Note that this is often observed with variational techniques, where model uncertainty is under-estimated.
This is because all variational inference techniques would under-estimate model uncertainty unless the true posterior is in the class of approximating variational distributions. Note also that the models correspond to different GP covariance functions that would result in different variance estimations.

\end{document}